\documentclass[12pt]{alt2023} 

\title[Linear RL with Ball Structure Action Space]{Linear Reinforcement Learning with Ball Structure Action Space}
\usepackage{times}
\usepackage{amsmath, amsfonts, amssymb}
\usepackage{algorithm, algpseudocode}

\newcommand{\mS}{\mathcal{S}}
\newcommand{\mA}{\mathcal{A}}

\newcommand{\EE}{\mathbb{E}}
\newcommand{\RR}{\mathbb{R}}
\newcommand{\argmax}{\mathop{\arg\max}}

\makeatletter
\newcommand\footnoteref[1]{\protected@xdef\@thefnmark{\ref{#1}}\@footnotemark}
\makeatother

\newtheorem{assumption}{Assumption}

\altauthor{
 \Name{Zeyu Jia}\footnote{This work was done while Zeyu interned at Amazon.\label{refnote}} \Email{zyjia@mit.com}\\
 \addr Massachusetts Institute of Technology
 \AND
 \Name{Randy Jia} \Email{randyjia@amazon.com}\\
 \addr Amazon
 \AND
 \Name{Dhruv Madeka} \Email{maded@amazon.com}\\
 \addr Amazon
 \AND
 \Name{Dean P. Foster} \Email{foster@amazon.com}\\
 \addr Amazon
}

\begin{document}

\maketitle

\begin{abstract}%
	We study the problem of Reinforcement Learning (RL) with linear function approximation, i.e. assuming the optimal action-value function is linear in a known $d$-dimensional feature mapping. Unfortunately, however, based on only this assumption, the worst case sample complexity has been shown to be exponential, even under a generative model. Instead of making further assumptions on the MDP or value functions, we assume that our action space is such that there always exist playable actions to explore any direction of the feature space. We formalize this assumption as a ``ball structure'' action space, and show that being able to freely explore the feature space allows for efficient RL. In particular, we propose a sample-efficient RL algorithm (BallRL) that learns an $\epsilon$-optimal policy using only $\tilde{\mathcal{O}}\left(\frac{H^5d^3}{\epsilon^3}\right)$ number of trajectories.
\end{abstract}

\begin{keywords}%
	Markov Decision Process, Reinforcement Learning
\end{keywords}

\section{Introduction}
\par Reinforcement Learning (RL) is a well-studied framework for sequential decision making that has been successfully applied to real-world problems in fields such as game-play (Atari, AlphaGo, Starcraft), robotics, operations management, and more~\citep{mnih2013playing, silver2016mastering, vinyals2017starcraft, kober2013reinforcement}. However, many of the existing theoretical results can not be applied to many practical applications due to intractably-large number of states and/or actions. A common modeling assumption to address this issue is to assume the existence of a known feature mapping that maps state and action to a $d$-dimensional feature vector, and that either the underlying MDP dynamics or value functions is linear in this feature mapping. In this work, we consider the common setting where the optimal action-value function (or $Q^*$-function) is linear and can be written as the inner product of the feature mapping of state-action pairs and some unknown parameter vector. The primary goal is to determine whether there exists algorithms that can achieve a near-optimal policy using an efficient number of samples. Here, efficient sample complexity refers to a polynomial number of samples with respect to the feature dimension $d$, the horizon $H$, and the size of the action set $|\mA|$.

\par This setting has garnered much attention recently, however, in the general case pessimistic results have been shown in~\citet{weisz2021exponential, weisz2022tensorplan, du2019good, wang2021exponential}, which indicate that this problem is exponentially hard in the horizon $H$ or the size of the action set $|\mA|$. Furthermore, this pessimistic result is true even with access to a generative model that allows for arbitrary state ``resets.'' Recently, several works have made further assumptions on the MDP that allow for efficient learning when the $Q^*$ function is linear~\citep{jin2020provably, amortila2022few}. These typically include additional assumptions on the transition or reward model, or access to additional side information such as expert queries. However, many of these assumptions are restrictive, unrealistic, or unfeasible for many practical use cases, since in the real world, we typically do not we have well-behaved transition models or access to expert oracles.

\par We seek a general yet practical assumption that is novel, realistic, and amenable to efficient learning. Our work is motivated by the observation that, in some difficult real-world RL applications such as game-play and operations management, it may be easier to think of actions (or consecutive actions) in feature space rather than state space. For example, in a typical dungeon-survival game with various tasks such as fighting monsters, eating food, or searching for treasure, the feature space could include combat statistics, health, and special items. Instead of actions consisting of low-level controls (e.g. movement, engage, run, etc.), we would consider higher-level ``feature space'' actions (e.g. fight monster, eat food, dig for treasure). Now, we conjecture that if a learning algorithm is always able to play actions to property explore the feature space, then, combined with a $d$-dimensional feature mapper that exists in the case of linear RL, it should be able to learn a near-optimal policy efficiently. In order to mathematically characterize this property, we introduce the concept of a ``ball structure'' action space. This assumes that our action space always lies within a $d$-dimensional ball of radius $\rho$, so that every direction of the feature space has a corresponding action that can be taken, and therefore at any time step, we are able to explore in any direction of the feature space. However, a perfect ball-shaped action space may be somewhat unrealistic, therefore, we allow some flexibility on the degree of exploration in each direction by considering less-restrictive settings, such as when the action is instead contained within a convex set, or when the radius of the ball is allowed to differ from one time step to the next. 

\par Our main result is the BallRL (pronounced \emph{baller}) algorithm that leverages the exploration capabilities of the ball structure assumption and achieves sample-efficient bounds on learning. The results hold under very mild trajectory learning/PAC learning setting, i.e. we do not assume have access to the action space, and each sampled trajectory gives us information only about action sets along the trajectory, together with total rewards. The algorithm takes advantage of the ball structure action space for exploration, which can be shown to be efficient using the closed form solution of the optimal Bellman Equation, and enjoys a sample complexity bound of $\tilde{\mathcal{O}}(H^5d^3/\epsilon^2)$ for an $\epsilon$-optimal policy. Furthermore, a similar algorithm and complexity bound hold in the case of convex action set instead of a ball action set, under additional mild assumptions. All together, our results show that with a ball structure action set, we can achieve an exponential improvement in comparison to algorithms to linear $Q^*$ problem without the ball structure assumption. We also demonstrate that our algorithm is easy to implement and is computationally efficient as well.

\subsection{Organization of the Paper}
\par The rest of the paper is structured as follows: in Section \ref{sec-background} and \ref{sec-related} we will introduce the problem setting and review prior work in the literature for the linear RL problem. In Section \ref{sec-alg} we present our learning algorithm where we demonstrate that efficient learning is possible assuming a ball structure action set. In particular, we present two special generalizations of the assumption: in Section \ref{sec-warm} consider when every state in step $h$ shares the same convex action set, and in Section \ref{sec-different} we assume the ball structure action set is allowed to vary by state. Note that the simpler ball structure assumption is a special case of both settings. We finally conclude in Section \ref{sec-summary} with some discussion.

\subsection{Notations}
\par We will use $\langle\cdot, \cdot\rangle$ and $\|\cdot\|_2$ to denote the inner product and the 2-norm in $\RR^d$, respectively. Let $B_2(\rho) = \left\{x\in\RR^d|\|x\|_2\le \rho\right\}$ represent the $L_2$-ball of radius $\rho$ in $\RR^d$. The expectation $\mathbb{E}^\pi$ will denote the expectations over all trajectories obtained according to $\pi$ and the underlying transition models and reward functions. We also follow standard big-Oh notation, that is, we will write $A = \mathcal{O}(B)$ if there exists some positive constant $c$ such that $A\le cB$, and write $A = \tilde{\mathcal{O}}(B)$ if there exists some $c = \mathrm{polylog}(d, H, \delta, 1/\epsilon)$ such that $A\le cB$. Here, $d$ is the dimension of the feature space, $H$ the time horizon of an episode, $\delta$ the high probability parameter, and $\epsilon$ the  near-optimality parameter of the learned policy.

\section{Background}\label{sec-background}
\subsection{Preliminaries}
\par A Markov Decision Process~\citep{sutton2018reinforcement, puterman2014markov} is a well-known model of the typical reinforcement learning environment. We consider finite-horizon MDPs which are defined by the tuple $\mathcal{M} = (\mS, \mA, P, H,  r, \mu)$, where the horizon $H\in \mathbb{N}$ and the state space $\mS = \mS_1\cup\cdots\cup \mS_H$ is known to the learner, but the action space $\mA(s)$ of each state $s\in\mS$, the transition model $P: \mS\times\mA\to \mS$, the reward function $r: \mS\to \RR$, and initial state distribution $\mu$ are not known. To avoid confusion, without loss of generality we assume that $\mS_1, \cdots, \mS_H$ have no intersection between each other.
\par For a given MDP, a policy $\pi: \mS \to \mA$ is a mapping from state space to the action space, where $\pi(s)\in\mA(s)$ for all $s\in\mS$. For a given policy $\pi$, we define its value functions ($V$) and $Q$ functions according to the following iterative equations:
\begin{align*}
	V_{H+1}^\pi(s_{h+1}) & = 0, \quad Q_{H+1}(s_{H+1}, a_{H+1}) = 0, \quad \forall s_{H+1}, a_{H+1},\\
	Q_h^\pi(s_h, a_h) & = r(s_h, a_h) + \sum_{s_{h+1}} P(s_{h+1}|s_h, a_h)V_{h+1}^\pi(s_{h+1}),\\
	V_h^\pi(s_h) & = Q_h^\pi(s_h, \pi(s_h)).
\end{align*}
We further define the optimal $Q$ and $V$ function as:
\[ Q_h^*(s_h, a_h) = \max_\pi Q_h^\pi(s_h, a_h), \quad V_h^*(s_h) = \max_\pi V_h^\pi(s_h).\]
For the optimal $Q^*$ function we have the optimal Bellman Equations, so that for all $1\le h\le H, s_h\in\mS_h, a_h\in \mA(s_h)$,
\begin{equation}\label{eq-bellman}
Q_h^*(s_h, a_h) = r(s_h, a_h) + \sum_{s_{h+1}} P(s_{h+1}|s_h, a_h)\max_{a_{h+1}}Q_{h+1}^*(s_{h+1}, a_{h+1}).\end{equation}
A typical reinforcement learning problem objective is to determine an algorithm that recovers a policy $\pi$ that performs well relative to the unknown optimal policy $\pi^*$; performance is generally defined by comparing the learned policy's and optimal policy's value functions. In the following section we detail our specific problem setting and objective.

\subsection{Our problem setting}
\par Due to the intractability of dealing with extremely high-dimensional state spaces, we make the standard assumption of a linear $Q^*$ function for our problem; that is, the optimal $Q^*$ function is linear in a $d$-dimensional feature mapping of the state and action:
\begin{assumption}[Non-Stationary Linear $Q^*$ Assumption]\label{ass-transition}
	For each state-action pair $(s, a)$, there exists a feature vector $\varphi(s, a)\in \RR^d$. There are also $H$ unknown parameters $\theta_1^*, \cdots, \theta_H^*$, such that the $Q^*$-function of state-action pairs has the following parametrization
	\[Q_h^*(s_h, a_h) = \langle \varphi(s_h, a_h), \theta_h^*\rangle, \quad \forall 1\le h\le H, s_h\in\mS_h, a_h\in \mA(s_h).\]
\end{assumption}
While Assumption~\ref{ass-transition} appears to be a very strong statement on the optimal $Q$ function, it is known that by itself, the assumption is not enough to guarantee efficient learning. Therefore, we present our ball structure assumption that we will show will allow for sample-efficient RL under the linear $Q^*$ assumption.
\begin{assumption}[Ball Structure Action Set]\label{ass-ball}
	Define the $L_2$ ball with radius $\rho > 0$ as
	\[B_2(\rho)\triangleq \left\{x\in\RR^d\big|\|x\|_2\le \rho\right\}.\]
	For each state $s$, there a feature vector $\varphi(s)\in \RR^d$ and a positive number $\rho(s)$ such that
	\[\left\{\varphi(s, a)|a\in\mA_s\right\} = \varphi(s) + B_2(\rho(s)).\]
\end{assumption}

\begin{remark}
	Without loss of generality, we can assume that
	$$\mA(s) = B_2(\rho_h(s))\triangleq \left\{a\in\RR^d\big|\|a\|_2\le \rho_h(s)\right\},$$
	and also
	$$\varphi(s, a) = \varphi(s) + a.$$
This is because if there exists two actions $a_1, a_2$ such that $\varphi(s, a_1) = \varphi(s, a_2)$, then Assumption~\ref{ass-transition} implies that $Q^*(s, a_1) = Q^*(s, a_2)$. Hence if we remove $a_2$ from the action set, the value of $V^*(s_1)$ will remain the same. Therefore, if we remove these redundant actions and find a near-optimal policy of the MDP, this policy must also be a near-optimal policy of the original MDP.
	\par By above, after removing redundant actions, we can assume that $a\to \varphi(s, a)$ is an injection, meaning  $a\to \varphi(s, a)$ is a one-to-one mapping from $\mA(s)$ to $B_2(\rho(s)) + \varphi(s)$. Hence we can replace every action $a$ with $\varphi(s, a) - \varphi(s)$, and  then we will have property $\varphi(s, a) = \varphi(s) + a$. Thus, without loss of generality, in the rest of the paper, we will assume $\varphi(s, a) = \varphi(s) + a$ always holds.
\end{remark}

\par Because the transition model and reward function are unknown at the beginning, the learner will only be able to access samples, or realizations, of them by directly interacting with the environment. That is, the learner must execute a policy to actually observe the outcome of those actions. We will consider the following \emph{trajectory learning} setting:

\begin{definition}[Trajectory Learning]\label{def-trajectory}
At every iteration, the learner first picks a policy (a function mapping every state $s\in\mS$ to some action in $\mA(s)$), and then a trajectory $(s_1, s_2, \cdots, s_H)$ is sampled according to the true underlying MDP. Only the following two pieces of information are revealed to the learner:
\begin{enumerate}
	\item $\mA(s_h)$: The action sets of each state in the trajectory;
	\item $\sum_{h=1}^H R(s_h, a_h)$: The sum of total reward of the trajectory, where $R(s_h, a_h)$ denotes the instant reward obtained by taking action $a_h$ at state $s_h$, which satisfies $\EE[R(s_h, a_h)] = r(s_h, a_h)$.
\end{enumerate}
\end{definition}
\begin{remark}
Note that our trajectory learning setting is weaker than the standard PAC learning setting in the literature, where it is assumed that all the information of the trajectories is revealed, including the states $s_1, \cdots, s_H$ and the instantaneous rewards $R(s_h, a_h)$. Our algorithm also does not require the use of a generative model that is standard in some linear $Q^*$ works. Therefore, our algorithm applies to both the common PAC learning setting and generative model setting. For more information about this, please refer to Section \ref{sec-related}.
\end{remark}

\par Finally, in order to measure the performance of our learner's policy, we define the closeness to optimality of a policy via the standard notion of an $\epsilon$-optimal policy:
\begin{definition}[$\epsilon$-optimal policy]
	If a policy $\pi$ satisfies 
	\[|V^\pi(s_0) - V^*(s_0)|\le \epsilon,\]
	then we call the policy $\pi$ an $\epsilon$-optimal policy. Here, $V^{\pi}$ is the value function with respect to following the policy $\pi$, and $V^*$ is the value function of the true optimal policy.
\end{definition}
\par Our objective in this work is to develop an algorithm which can find an $\epsilon$-optimal policy with high probability, by using a polynomial number (in $d, H$ and $1/\epsilon$) of trajectory learning iterations.

\section{Related Literature}\label{sec-related}
\par The linear $Q^*$ problem is one of the simplest and most intuitive ways to describe reinforcement learning with parametrization. Many works have studied this setting of RL with the goal to develop a sample-efficient algorithm to learn a near-optimal policy. However, in the most general case, recent work has yielded only pessimistic results related to this problem. In~\cite{weisz2021exponential, weisz2022tensorplan, du2019good, wang2021exponential, foster2021statistical}, the linear $Q^*$ problem has been shown to be exponentially hard in $d$ or $H$ or $|\mA|$, even when the number of actions are small. Their main idea revolves around showing a lower exponential bound by constructing a needle in haystack-type MDP, i.e., among exponentially many actions there is only one action that induces rewards, hence in order to find the optimal action the learner must run policies an exponential number of times. Additionally, they also adopt the Johnson-Lindenstrauss lemma to show that they can choose these actions such that every two actions are sufficiently far away from each other, so that querying non-optimal actions gives limited information of the optimal action.

\par Apart from pessimistic results, there are many works which demonstrate that the linear $Q^*$ problem is polynomially solvable with added additional assumptions. Assumptions are quite varied and numerous, and we attempt to give an overview of the different types that have allowed for efficient learning. If for all policies $\pi$, the $Q$-function $Q^\pi$ can be linearly parameterized, then the problem is polynomially solvable by using approximate policy iteration~\citep{lattimore2020learning}. If both the transition model and reward function are deterministic, then the problem is polynomially solvable by eliminating functions that does not satisfy the linear $Q^*$ function assumptions~\citep{wen2013efficient}. If a `core set' (that is, features of every state action pairs can be written as the convex combinations of features in the core set) exists for the MDP, then the problem is polynomially solvable~\citep{zanette2019limiting, shariff2020efficient}. In comparison to our assumption, our algorithm has access to an orthogonal basis at first, which is similar to the idea of core set. However, the core set cannot capture our setting, since the ball cannot be written as convex combination of basis vectors - simply adopting their algorithm would induce exponential sample complexity. Under the assumption that the action set is finite,  the TensorPlan Algorithm in ~\citet{weisz2021query} can obtain an $\epsilon$-optimal policy using $\mathrm{poly}\left(\left(\frac{dH}{\epsilon}\right)^{|\mA|}\right)$ number of samples. Alternatively, if we assume access to an expert oracle which gives the value of $Q^*(s, a)$ when queried at state $(s, a)$, the DELPHI algorithm can solve these linear $Q^*$ problem in polynomial time using no more than $\mathcal{O}(d)$ calls of expert queries ~\citep{amortila2022few}.

\par Beyond the linear $Q^*$ problem, there are also several works which achieve polynomial sample complexity under general assumptions of the MDP's underlying properties. If the transition model can be linearly parametrized, then the MDP problem becomes polynomially solvable as shown in~\cite{jin2020provably, yang2019sample, yang2020reinforcement, jia2020model}. However, a linear transition model is a fairly strong assumption and generally not a very practical assumption, as most systems do not behave as such. There are also works focused on generalized function approximations, e.g. Eluder Dimensions~\citep{ayoub2020model, wang2020reinforcement}, Bellman Rank~\citep{jiang2017contextual}, Bellman Eluder Dimension~\citep{jin2021bellman}, Bilinear Class~\citep{du2021bilinear}, Bellman Closeness~\citep{jin2021bellman, zanette2020learning}. However, again these assumptions on the models are either hard to verify in practice or generally do not occur in real world systems, which makes the use of these algorithms difficult to justify in practice.

\section{The BallRL Algorithm}\label{sec-alg}
We now present the main result of our paper, that is, an algorithm that achieves polynomial sample complexity in the linear $Q^*$ setting under the assumption of a ball structure action space. Before proceeding with the details, we highlight two versions of our algorithm, Convex-BallRL and DiffR-BallRL, both of which are essentially extensions of the standard ball structure assumption (Assumption \ref{ass-ball}). In the first case, we consider convex action sets where the action sets are identical across state. While every state necessarily has the same set of actions to take, the magnitude to which one can explore different directions is permitted to vary, so long the overall action set is convex. This can be seen as a slightly more realistic version of the standard ball assumption, as in practice it may be difficult to guarantee the magnitude of every feature direction to be the same. In the second case, we consider the standard ball structure action set but allow the action set to vary depending on the current state. The motivation behind these two slightly different settings is to represent a more realistic generalization of the original ball structure presented earlier, as in practical settings action spaces may not always be uniformly a perfect ball.

\subsection{Identical Convex Action Sets within One Step}\label{sec-warm}
\par In this section, we make the assumption that the action set $\mA(s_h)$ is identical for every $s_h\in\mS_h$, and moreover, we assume the action sets are regular convex sets, which is a generalization of the ball structure presented in Assumption \ref{ass-ball}. Intuitively, the action set is contained between a smaller radius and a larger radius ball.
\begin{definition}[Regular Convex Set]
	We call a set $\mathcal{M}\subset \mathbb{R}^d$ a regular convex set with parameter $B\ge 1$ if there exists $\eta \ge \rho > 0$ such that $\frac{\eta}{\rho} = B$ and
	$$B_2(\rho)\subset \mathcal{M}\subset B_2(\eta).$$
\end{definition}
\begin{remark}
	Regular convex sets include many different types of structures such as balls, cubes, ellipsoids, etc. Some specific examples include:
	\begin{enumerate}
		\item All balls are regular convex sets with parameter $1$;
		\item Cubes in $d$ dimension are regular convex sets with parameter $\sqrt{d}$;
		\item Ellipsoids are regular convex sets with parameter $\frac{a_{\max}}{a_{\min}}$, where $a_{\max}, a_{\min}$ are the longest and shortest axes.
	\end{enumerate}
\end{remark}

\par Let us formally characterize our assumption for the setting with convex action sets.
\begin{assumption}[Identical Convex Action Sets within One Step]\label{ass-convex}
	For every $1\le h\le H$, there exists a regular convex set $\mA_h$ with parameter $B$, such that for all $s_h\in\mS_h$, $\mA(s_h) = \mA_h$. Specifically, there exists $\rho_1, \cdots, \rho_H, \eta_1, \cdots, \eta_H$, such that for every $1\le h\le H$ we have 
	$$\frac{\eta_h}{\rho_h} = B, \qquad B_2(\rho_h)\subset \mA_h\subset B_2(\eta_h).$$
	Without loss of generality, we also assume that the features still satisfy $\varphi(s, a) = \varphi(s) + a$.
\end{assumption}
\par We develop an algorithm, Convex-BallRL, that works in the trajectory learning setting (Definition~\ref{def-trajectory}) under Assumption \ref{ass-transition} and \ref{ass-convex}, and is guaranteed to find an $\epsilon$-optimal policy using a polynomial number of trajectories.

\subsubsection{Intuition and key ideas}
\par Before presenting the algorithm itself, we provide some intuition on the key ideas behind our algorithm. With loss of generality, we assume that we know the value of $\rho_1, \cdots, \rho_H$ at the beginning. Otherwise, we can run one trajectory according to any policy, then all the action sets $\mA_1, \cdots, \mA_H$ will be revealed to us, from which we can determine the values of $\rho_1, \cdots, \rho_H$.
\par We start by observing the following equation due to telescoping of Bellman Equation \eqref{eq-bellman}:
\begin{equation}\label{eq-warm}\mathbb{E}\left[\langle\varphi(s_1), \theta_1^*\rangle\right] + \mathbb{E}\left[\sum_{h=1}^H\langle a_h, \theta_h^*\rangle\right] = \sum_{h=1}^H\mathbb{E}[R(s_h, a_h)] + \sum_{h=1}^H \rho_{h+1}\max_{a_{h+1}\in\mathcal{A}_{h+1}}\langle a_{h+1}, \theta_{h+1}^*\rangle.\end{equation}
\par Our next observation is that the first term of LHS and the second term of RHS in \eqref{eq-warm} are identical for every policy. Hence, if we compare \eqref{eq-warm} between two different policies, we can obtain information of $\theta_h^*$ (according to $\langle a_h, \theta_h^*\rangle$) based on the first term in RHS, which can be estimated through sampled trajectories. Formally, we choose $\pi_0$ to be the all-zero policy: 
\begin{equation}\label{eq-pi0}\pi_0(s_h) = \mathbf{0}\in\RR^d, \quad \forall 1\le h\le H,\end{equation}
and $\pi_{h, i}$ ($1\le h\le H, 1\le i\le d$) to be the following policy: for $1\le h'\le H$ and $s_{h'}\in \mS_{h'}$,
\begin{equation}\label{eq-pihi}\pi_{h, i}(s_{h'}) = \begin{cases} \mathbf{0} &\quad \text{if }h'\neq h,\\ \rho_h\mathbf{e}_i&\quad \text{if } h' = h\end{cases},\end{equation}
where $\mathbf{e}_i$ is the $i$-th basis vector in $\RR^d$. Comparing \eqref{eq-warm} according to policy $\pi_0$ and also policy $\pi_{h, i}$, we obtain that
$$\rho_h\langle \mathbf{e}_i, \theta_h^*\rangle = \mathbf{E}^{\pi_{h, i}}\left[\sum_{h=1}^H R(s_h, a_h)\right] - \mathbf{E}^{\pi_{0}}\left[\sum_{h=1}^H R(s_h, a_h)\right].$$
The right hand side can be estimated according to trajectories from policy $\pi_0$ and $\pi_{h, i}$, which leads to the estimate of $i$-th component of $\theta_h^*$.
\par Finally after getting accurate enough estimations $\hat{\theta}_h\in \RR^d$ on $\theta_h^*$, we adopt the greedy policy, i.e.
\begin{equation}\label{eq-policy}\pi(s_h) = \argmax_{a_h\in \mA_h} \langle a_h, \hat{\theta}_h\rangle,\end{equation}
and then show that this policy is a nearly optimal policy. 

\subsubsection{Algorithm and Sample Complexity}
\par The pseudocode for BallRL with convex action sets is given in Algorithm \ref{alg-convex}. The main result of this section is the following theorem about its sample complexity, in particular, that it has polynomial sample complexity. The complete proof details are provided in Appendix \ref{app-convex}.
\begin{theorem}\label{thm-convex}
	For any $\delta > 0$, if we choose 
	$$M = \frac{8H^2B^2d\log(2dH/\delta)}{\epsilon^2},$$
	then with probability at least $1 - \delta$, the output policy from the above algorithm is an $\epsilon$-optimal policy. The total number of trajectories used in this algorithm is
	$$\frac{16H^3B^2d^2\log(2dH/\delta)}{\epsilon^2}.$$
\end{theorem}
\begin{remark}
	If we assume all action sets have ball structure, then all action sets are regular convex sets with parameter $1$. Hence the above algorithm is guaranteed to find an $\epsilon$-optimal policy using $\tilde{\mathcal{O}}\left(\frac{H^3d^2}{\epsilon^2}\right)$ number of trajectories.
\end{remark}

\begin{algorithm}
	\caption{Convex-BallRL} \label{alg-convex}
	\begin{algorithmic}
		\State \textbf{Input: } $d, H, M, \rho_h$;
		
		\For{$h = 1:H$}{
			\For{$i = 1:d$}{
				\State Choose policy $\pi_0$ such that $\pi_0(s_{h'}) = \mathbf{0}$ for any $s_{h'}, h'\in [H]$.
				\State Choose policy $\pi_{h, i}$ such that $\pi_{h, i}(s_{h'}) = \mathbf{0}$ if $h'\neq h$ and $\pi_{h, i}(s_h) = \rho_h\mathbf{e}_i$ for all $s_h$.
				\State Get $M$ trajectories from policy $\pi_0$ and $\pi_{h, i}$ respectively, and calculate the average of total rewards: $R_0$ and $R_{h, i}$.
			}
			\State Let $\hat{\theta}_{h}$ to be the vector whose $i$-th component to be $\frac{R_{h, i} - R_0}{\rho_h}$.
			\State Let $a_h = \arg\max_{a\in\mathcal{A}_h} \langle a_h, \hat{\theta}_h\rangle.$
		}
		\State \textbf{Output: } Policy $\pi$: $\pi(s_h) = a_h$ for any $s_h, h\in [H]$.
	\end{algorithmic}
\end{algorithm}

\subsection{Different Radius}\label{sec-different}
\par In this section, we abandon the assumption that all states $s_h$ in step $h$ share identical action sets, and allow for the action set to vary depending on the state. However, we again assume that the action set corresponding to each state is a ball as in Assumption \ref{ass-ball}. We further assume that the norm of $\theta_h$ are all the same for $1\le h\le H$, and also that the norm of features, rewards and radius are bounded:
\begin{assumption}[Boundedness]\label{ass-bound}
	For each state $s\in\mS$, action $a\in\mA(s)$, we have
	$$\|\varphi(s, a)\|_2\le 1;$$
	For some $\Theta\in [0, 1]$, we have
	$$\|\theta_1\|_2 = \cdots = \|\theta_H\|_2 = \Theta;$$
	For every trajectories $(s_1, a_1, \cdots, s_H, a_H)$, we have
	$$0\le \sum_{h=1}^H R(s_h, a_h)\le 1,\quad 0\le \sum_{h=1}^H \rho(s_h)\le 1.$$
\end{assumption}
\par We again aim to develop an algorithm that works under Definition \ref{def-trajectory} (trajectory learning), but under Assumption \ref{ass-transition} (Linear $Q^*$ assumption), \ref{ass-ball} (Ball Structure Assumption) and \ref{ass-bound} (Boundedness Assumption).

\subsubsection{Intuition and key ideas}
\par We begin by presenting the following key ideas of our algorithm:
\paragraph{To Exploit the Ball Structure Action space} Similar to (\ref{eq-warm}) in Convex-BallRL, our algorithm is again based on the telescoping of Bellman Equation \eqref{eq-bellman}, which exploits the ball structure of the action space:
\begin{equation}\label{eq-bell}
\langle\varphi(s_1), \theta_1^*\rangle + \mathbb{E}^\pi\left[\sum_{h=1}^H\langle a_h, \theta_h^*\rangle\right] = \mathbb{E}^\pi\left[\sum_{h=1}^HR(s_h, a_h)\right] + \Theta\cdot \mathbb{E}^\pi\left[\sum_{h=1}^H \rho(s_{h+1})\right].\end{equation}

\paragraph{Estimation of Norm by Grid Search}
\par According to \eqref{eq-bell}, we can estimate $\theta_1^*, \cdots, \theta_H^*$ in the LHS based on the RHS. However, $\Theta$, which is the norm of the unknown parameters $\theta_1^*, \cdots, \theta_H^*$, is difficult to estimate. Hence in our algorithm, we adopt a grid search method for the value of $\Theta$: choosing $\xi = l\varepsilon$ for $1\le l\le \frac{1}{\varepsilon}$, so that at least one such $\xi$ is $\varepsilon$-close to the true $\Theta$. Therefore, if we develop our policy based on these $\xi$, then at least one policy will necessarily be an $\epsilon$-optimal policy.

\paragraph{Hierarchical Exploration} For the exploration in our algorithm, we will choose  actions to be $\rho(s_h)\mathbf{e}_i$ for $1\le i\le d$ in order to give information about the $i$-th component of $\theta_h^*$. However, one problem is that this estimation has accuracy at most $1/\rho(s_h)$, which will explode as $\rho(s_h)$ goes to zero. To deal with this problem, we consider a hierarchical exploration method: 
\par Suppose the policy we currently use for exploration is $\pi_e$, and the greedy policy we calculated is $\pi$. We can show that the exploration will guarantee $1/({\mathbb{E}^{\pi_e}[\rho(s_h)]}\sqrt{M})$ accuracy on $\theta_h^*$ (up to logarithmic factors), and hence the error of $\pi$ is $\mathbb{E}^{\pi}[\rho(s_h)]/({\mathbb{E}^{\pi_e}[\rho(s_h)]}\sqrt{M})$. Therefore, if for every $1\le h\le H$ we all have $\mathbb{E}^{\pi}[\rho(s_h)]\le 2{\mathbb{E}^{\pi_e}[\rho(s_h)]}$, then the error of the greedy policy is of order $2/\sqrt{M}$, which can be bounded by choosing some proper $M$. Otherwise, we use the greedy policy $\pi$ to construct another exploration policy as follows:
\begin{equation}\label{eq-pi}
	\begin{aligned}
		\pi_{h, 0}(s_{h'}) & = \begin{cases}
			\pi(s_{h'}) &\quad \text{if } 1\le h' < h;\\
			0 &\quad \text{if }h'\ge h;
		\end{cases}\\
		\pi_{h, i}(s_{h'}) & = \begin{cases}
			\pi(s_{h'}) &\quad \text{if } 1\le h' < h;\\
			\rho(s_h)\mathbf{e}_i &\quad \text{if }h' = h;\\
			0 &\quad \text{if }h'\ge h.
		\end{cases}
	\end{aligned}
\end{equation}
Then these new policies $\pi_{h, i}$ will guarantee that $\mathbb{E}^{\pi_{h, i}}[\rho(s_h)]\ge 2\mathbb{E}^{\pi_e}[\rho(s_h)]$, i.e. the value of $\mathbb{E}^{\pi_e}[\rho(s_h)]$ becomes at least twice of its previous value. Therefore, we can show that this process will end in at most $H\log(1/\varepsilon)$ number of times, provided that the initial value of $\mathbb{E}^{\pi_e}[\rho(s_h)]$ is at least $\varepsilon$. 

\paragraph{Ignore Small Radius} We will show that if within a policy $\pi$, the expected radius $\EE^\pi[\rho(s_h)]$ at step $h$ is smaller than $\varepsilon$, then the effect of different actions within this step can be ignored, and we do not need to carry out the above exploration in this step.

\subsubsection{Algorithm and Sample Complexity}
\par Combine these ideas together, we construct the following Algorithm \ref{alg-diff}.
\begin{algorithm}[H]
	\caption{DiffR-BallRL}
	\label{alg-diff}
	\begin{algorithmic}
		\State \textbf{Input: } $d, H, M_1, M_2, \varepsilon, \eta, L = 1/\eta$;
		\State Let $\rho_h = 0, \rho_h^l = 1$ for all $1\le h\le H, 1\le l\le L$.
		\State Let policy $\pi_l'$ to be the policy such that $\pi(s_h) = 0$ for all $s_h\in \mS_h$ and $1\le h\le H$ and $1\le l\le L$.
		
		\While{$\exists 1\le h\le H$ and $1\le l\le L$ such that $\rho_h^l\ge 2\rho_h$ and $\rho_h^l\ge \varepsilon$}{
		\State Fix $h, l$ to be the one that $\rho_h^l\ge \rho_h$, and construct Policy $\pi_{h, 0}$ and $\pi_{h, i}$ ($1\le i\le d$) based on policy $\pi_l'$ according to \eqref{eq-pi} (use $\pi = \pi_l'$ in \eqref{eq-pi}).
		\State Collect $M_1$ trajectories according to policy $\pi_{h, i}$ for $0\le i\le d$ each, and calculate the average of total reward $\sum_{h'=1}^H R(s_h, a_h)$, and average of $\sum_{h=2}^H \rho(s_h)$ as $R_{h, i}$ and $s_{h, i}$, respectively.
		\State $\rho_h\leftarrow \rho_h^l.$
		
		\For{$l=1:L$}{
			\State Let $\xi = l\eta$.
			\State Calculate $\hat{\theta}_{h, i}^l = \frac{(s_{h, i} - s_{h, 0})\xi + R_{h, i} - R_{h, 0}}{\rho_h}$ for $1\le i\le d$.
		}
		\State Let $\hat{\theta}_h^l = \sum_{i=1}^d\hat{\theta}_{h, i}^l \mathbf{e}_i$.
			\State Construct policy $\pi_l'$:
			$$\pi_l'(s_{h'}) = \argmax_{a_{h'}\in B_2(\rho(s_{h'}))} \left\langle a_{h'}, \hat{\theta}_{h'}^l\right\rangle, \quad \forall s_{h'}\in \mS_{h'}, 1\le h'\le H.$$
			\State For $1\le l\le L$, run policy $\pi_l'$ each for $M_2$ times, and calculate the average total reward $R_1, \cdots, R_L$, and also calculate the average of $\rho(s_{h'})$ as $\rho_{h'}^1, \cdots, \rho_{h'}^L$ for $1\le h'\le H$. 
			\State Let $l = \argmax R_l$ and $\pi = \pi'_l$.
		}
		\State \textbf{Output: }Policy $\pi$.
	\end{algorithmic}
\end{algorithm}

\par Finally, we arrive at our main result - that DiffR-BallRL is sample efficient. The proof details are provided in Appendix \ref{app-diff}.
\begin{theorem}\label{thm1}
	For any $0 < \delta < 1$, with the choice 
	\begin{align*}
		& \varepsilon = \frac{\epsilon}{8H}, \quad \delta' = \frac{\delta}{(d + 3HL)(1 + H\log_2(1/\varepsilon))},\quad \eta = \frac{\epsilon}{8Hd},\\
		& M_2 = 2\log(1/\delta')\cdot \frac{16(2 + 4H + 2Hd)^2}{\epsilon^2},\quad M_1 = 2\log(1/\delta')\cdot \frac{256H^2d^2}{\epsilon^2}, \quad L = \frac{1}{\eta}\end{align*}
	Algorithm \ref{alg-diff} will output an $\epsilon$-optimal policy with probability at least $1 - \delta$. This algorithm will use at most
	$$\tilde{\mathcal{O}}\left(\frac{H^5d^3}{\epsilon^3}\right)$$
	number of trajectories.
\end{theorem}

\paragraph{Proof Sketch of Theorem \ref{thm1}.} Our first step of the proof is to use Bellman Equation to prove \eqref{eq-bell}:
$$\langle\varphi(s_1), \theta_1^*\rangle + \mathbb{E}^\pi\left[\sum_{h=1}^H\langle a_h, \theta_h^*\rangle\right] = \mathbb{E}^\pi\left[\sum_{h=1}^HR(s_h, a_h)\right] + \Theta\cdot \mathbb{E}^\pi\left[\sum_{h=1}^H \rho(s_{h+1})\right],$$
which can be obtained through telescoping the following closed form of Bellman Equation at step $h$:
$$\mathbb{E}^\pi\left[\langle \varphi(s_h), \theta_h^*\rangle + \langle a_h, \theta_h^*\rangle\right] = \mathbb{E}^\pi\left[R(s_h, a_h) + \rho(s_{h+1})\cdot\|\theta_{h+1}\|_2 + \langle \varphi(s_{h+1}), \theta_{h+1}^*\rangle\right].$$
\par Our second step is a result bounding the value function error of the greedy policies:
$$\mathbb{E}[V_1^*(s_1)] - \mathbb{E}[V_1^\pi(s_1)]\le 2\sum_{h=1}^H \mathbb{E}^\pi[\rho(s_h)]\cdot \left\|\hat{\theta}_h - \theta_h^*\right\|_2.$$
Therefore, if $\mathbb{E}^\pi[\rho(s_h)]$ is small for some $h$ (say less than $\epsilon$), then we can ignore this term, since it will never make big difference on the error. In the following, we assume that $\mathbb{E}^\pi[\rho(s_h)]\ge \epsilon$.
\par Next, we observe that there exists some $\Theta' = l_0\eta$ such that $|\Theta' - \Theta|\le\eta$. And for the iteration $l_0$, according to Hoeffding inequality we can get
$$\rho_h\left\|\theta_h - \hat{\theta}_{h}^{l_0}\right\|_2\le \eta d + 2d\sqrt{\frac{2\log (1/\delta')}{M_1}} + d\sqrt{\frac{2\log (1/\delta')}{M_2}}$$
with high probability, where $\rho_h$ is the expectation of $\rho(s_h)$ according to the exploration policy.
\par Finally, if $\rho_h^{l_0}$ of the greedy policy satisfies that $\rho_h^{l_0}\le \rho_h$, then the above inequality can guarantee that this greedy policy is near optimal. Otherwise, the value of $\rho_h$ will become twice as before according to our algorithm, and this process will terminate in $\log(1/\epsilon)$ number of iterations, since according to our assumption the initial value of $\rho_h$ is at least $\epsilon$, and $\rho_h$ cannot be large than 1.
\par Combining these steps together, we can show that the algorithm will end in certain number of iterations, and when the algorithm ends, it will output a near optimal policy with high probability.


\section{Conclusion}\label{sec-summary}
\par We presented the BallRL reinforcement learning algorithm that provides sample-efficient learning guarantees when the optimal action-value function is linear and actions exhibit a ball structure.  We further generalized the ball structure to both convex actions sets and changing ball radius between states. Our techniques demonstrate that there is hope for efficient learning in linear RL when actions can sufficiently explore the feature space. The ball structure assumption itself is a sufficient, but not fully necessary condition to ensure full exploration of the feature space. We believe that the idea of the action set allowing for sufficient exploration can be achieved (perhaps approximately) in many practical settings.

An interesting research direction is to dive deeper into assuming the actions lie in convex sets instead of a pure ball structure. While the problem can be solved when the action set is consistent between all actions, it remains to be shown if convex sets can vary between states. Additionally, with different radii, our algorithm is polynomially efficient when unknown parameters across the horizon share the same norm. It would be valuable to see whether this assumption can be removed and parameters allowed to have different norms.

\acks{We thank Philip Amortila for helpful discussions.}

\bibliography{alt2023-sample}

\newpage
\appendix

\section{Proof of Theorem \ref{thm-convex}}
\label{app-convex}
\begin{proof}[Proof of Theorem \ref{thm-convex}]
\par First of all, according to Bellman Equation \eqref{eq-bellman}, we have
\begin{align*}
	&\quad \langle \varphi(s_h), \theta_h^*\rangle + \langle a_h, \theta_h^*\rangle = Q_h^*(s_h, a_h)\\
	& = r(s_h, a_h) + \sum_{s_{h+1}} P(s_{h+1}|s_h, a_h) V_{h+1}^*(s_{h+1})\\
	& = r(s_h, a_h) + \sum_{s_{h+1}} P(s_{h+1}|s_h, a_h) \max_{a_{h+1}\in \mathcal{A}_{h+1}}\langle \varphi(s_{h+1}) + a_{h+1}, \theta_{h+1}^*\rangle\\
	& = r(s_h, a_h) + \sum_{s_{h+1}} P(s_{h+1}|s_h, a_h)\langle \varphi(s_{h+1}), \theta_{h+1}^*\rangle + \sum_{s_{h+1}} P(s_{h+1}|s_h, a_h)\rho_{h+1}\max_{a_{h+1}\in\mathcal{A}_{h+1}}\langle a_{h+1}, \theta_{h+1}^*\rangle\\
	& = \mathbb{E}\left[R(s_h, a_h) + \langle \varphi(s_{h+1}), \theta_{h+1}^*\rangle + \rho_{h+1}\max_{a_{h+1}\in\mathcal{A}_{h+1}}\langle a_{h+1}, \theta_{h+1}^*\rangle\Big|s_h, a_h\right],
\end{align*}
where $R(s_h, a_h)$ is the instant reward we obtain after choosing action $a_h$ from state $s_h$ ($R(s_h, a_h)$ has mean $r(s_h, a_h)$). Hence for a given fixed policy $\pi$, suppose a trajectory following this policy is $(s_1, a_1, \cdots , s_H, a_H)$, then we have
$$\mathbb{E}^\pi\left[\langle\varphi(s_h), \theta_h^*\rangle\right] + \mathbb{E}^\pi[\langle a_h, \theta_h^*\rangle] = \mathbb{E}^\pi[\langle \varphi(s_{h+1}), \theta_{h+1}^*\rangle] + \mathbb{E}^\pi\left[R(s_h, a_h)\right] + \rho_{h+1}\max_{a_{h+1}\in\mathcal{A}_{h+1}}\langle a_{h+1}, \theta_{h+1}^*\rangle.$$
Summing this up from $h = 1$ to $h = H$ and noticing that $\theta_{h+1}^* = 0$, we obtain that
$$\mathbb{E}^\pi\left[\langle\varphi(s_1), \theta_1^*\rangle\right] + \mathbb{E}^\pi\left[\sum_{h=1}^H\langle a_h, \theta_h^*\rangle\right] = \sum_{h=1}^H\mathbb{E}^\pi[R(s_h, a_h)] + \sum_{h=1}^H \rho_{h+1}\max_{a_{h+1}\in\mathcal{A}_{h+1}}\langle a_{h+1}, \theta_{h+1}^*\rangle.$$
With our choice of $\pi_0$ (the policy which choose action $0$ at any state and step), we obtain
$$\mathbb{E}^{\pi_0}\left[\langle\varphi(s_1), \theta_1^*\rangle\right] = \sum_{h=1}^H\mathbb{E}^{\pi_0}[R(s_h, a_h)] + \sum_{h=1}^H \rho_{h+1}\max_{a_{h+1}\in\mathcal{A}_{h+1}}\langle a_{h+1}, \theta_{h+1}^*\rangle.$$
We notice that for every policy $\pi$, $\mathbb{E}^\pi\left[\langle\varphi(s_1), \theta_1^*\rangle\right]$ are identical. Hence after subtracting the above two equations, we get
\begin{equation}\label{eq-difference}
	\mathbb{E}^\pi\left[\sum_{h=1}^H\langle a_h, \theta_h^*\rangle\right] = \mathbb{E}^{\pi}\left[\sum_{h=1}^H R(s_h, a_h)\right] - \mathbb{E}^{\pi_0}\left[\sum_{h=1}^H R(s_h, a_h)\right].
\end{equation}
With our choice of policy $\pi_{h, i}$, the above equation indicates that
$$\rho_h\theta_{h, i}^* = \rho_h\langle\mathbf{e}_i, \theta_{h}^*\rangle = \mathbb{E}^{\pi_{h, i}}\left[\sum_{h=1}^H R(s_h, a_h)\right] - \mathbb{E}^{\pi_0}\left[\sum_{h=1}^H R(s_h, a_h)\right],$$
where $\theta_{h, i}^*$ is the $i$-th component of $\theta_h^*$. 
According to our algorithm, we have $\rho_h\hat{\theta}_{h, i} = R_{h, i} - R_{h, 0}$, and we also have 
$$\mathbf{E}\left[R_{h, i} - R_{h, i}\right] = \mathbb{E}^{\pi_{h, i}}\left[\sum_{h=1}^H R(s_h, a_h)\right] - \mathbb{E}^{\pi_0}\left[\sum_{h=1}^H R(s_h, a_h)\right] = \rho_h\theta_{h, i}^*.$$
Therefore, according to Hoeffding inequality, with probability at least $1 - 2\delta'$ we have,
$$\rho_h\left|\hat{\theta}_{h, i} - \theta_{h, i}^*\right|\le 2\sqrt{\frac{\log(1/\delta')}{2M}} = \sqrt{\frac{2\log(1/\delta')}{M}},$$
if we assume that $\sum_{h=1}^H R(s_h, a_h)\in [0, 1]$ always holds. This indicates that with probability at least $1 - 2dH\delta'$, for every $1\le h\le H$, 
$$\rho_h\left\|\hat{\theta}_h - \theta_h^*\right\|_2\le \sqrt{\frac{2d\log(1/\delta')}{M}}.$$
\par Moreover, since $a_h = \arg\max_{a\in\mathcal{A}_h} \langle a_h, \hat{\theta}_h\rangle,$ we have
\begin{align*}
	\langle a_h, \theta_h^*\rangle - \langle a_h^*, \theta_h^*\rangle & = \langle a_h, \theta_h^* - \hat{\theta}_h\rangle + \langle a_h, \hat{\theta}_h\rangle - \langle a_h^*, \hat{\theta}_h\rangle + \langle a_h^*, \hat{\theta_h} - \theta_h^*\rangle\\
	& \le \langle a_h, \theta_h^* - \hat{\theta}_h\rangle + \langle a_h^*, \hat{\theta_h} - \theta_h^*\rangle\\
	& \le 2\eta_h\cdot \|\theta_h - \hat{\theta}_h\|_2\\
	& \le 2B\cdot \sqrt{\frac{2d\log(1/\delta')}{M}},
\end{align*}
where the first inequality is due to $a_h = \arg\max_{a\in\mathcal{A}_h} \langle a_h, \hat{\theta}_h\rangle$, and the second inequality is due to $a_h, a_h^*\in \mathcal{A}_h\subset B_2(\eta_h)$.
Therefore, according to \eqref{eq-difference} we have
$$V^*(s_1) - V^\pi(s_1)\le 2BH\sqrt{\frac{2d\log(1/\delta')}{M}}.$$
With our choice of $\delta' = \frac{\delta}{2dH}$ and $M = \frac{8H^2B^2d\log(2dH/\delta)}{\epsilon^2}$, we have with probability at least $1 - \delta$, the output policy $\pi$ satisfies that
$$V^*(s_1) - V^\pi(s_1)\le\epsilon.$$

\end{proof}

\section{Proof of Theorem \ref{thm1}}
\label{app-diff}
\par To prove the complexity of our algorithm, we first show the following several lemmas:
\begin{lemma}\label{lem-eq}
	For each policy $\pi$, we have the following equation
	\begin{equation}\label{eq-eq}\langle\varphi(s_1), \theta_1^*\rangle + \mathbb{E}^\pi\left[\sum_{h=1}^H\langle a_h, \theta_h^*\rangle\right] = \mathbb{E}^\pi\left[\sum_{h=1}^HR(s_h, a_h)\right] + \Theta\cdot \mathbb{E}^\pi\left[\sum_{h=1}^H \rho(s_{h+1})\right],\end{equation}
	where we use the notation that $\rho(s_{H+1}) = 0$. 
\end{lemma}
\begin{proof}
	According to the Bellman Equation \eqref{eq-bellman} at step $h$, we have
	\begin{align*}
	Q_h^*(s_h, a_h) & = r(s_h, a_h) + \sum_{s_{h+1}} P(s_{h+1}|s_h, a_h) V_{h+1}^*(s_{h+1})\\
	& = r(s_h, a_h) + \sum_{s_{h+1}} P(s_{h+1}|s_h, a_h) \max_{a_{h+1}\in \mA(s_{h+1})}\langle \varphi(s_{h+1}) + a_{h+1}, \theta_{h+1}^*\rangle\\
	& = r(s_h, a_h) + \sum_{s_{h+1}} P(s_{h+1}|s_h, a_h) \max_{a_{h+1}\in B_2(\rho(s_{h+1}))}\langle \varphi(s_{h+1}) + a_{h+1}, \theta_{h+1}^*\rangle\\
	& = r(s_h, a_h) + \sum_{s_{h+1}} P(s_{h+1}|s_h, a_h)\rho(s_{h+1})\|\theta_{h+1}^*\|_2 + \sum_{s_{h+1}} P(s_{h+1}|s_h, a_h)\langle \varphi(s_{h+1}), \theta_{h+1}^*\rangle\\
	& = \mathbb{E}^\pi\left[R(s_h, a_h) + \rho_{h+1}\|\theta_{h+1}^*\|_2 + \langle \varphi(s_{h+1}), \theta_{h+1}^*\rangle\big|s_h, a_h\right]\\
	& = \mathbb{E}^\pi\left[R(s_h, a_h) + \rho_{h+1}\cdot \Theta + \langle \varphi(s_{h+1}), \theta_{h+1}^*\rangle\big|s_h, a_h\right].
	\end{align*}
	Since 
	$$Q_h^*(s_h, a_h) = \langle \varphi(s_h, a_h), \theta_h^*\rangle = \langle \varphi(s_h), \theta_h^*\rangle + \langle a_h, \theta_h^*\rangle,$$
	if we take the expectation over $s_h$ and $a_h$, we will have
	$$\mathbb{E}^\pi\left[\langle \varphi(s_h), \theta_h^*\rangle + \langle a_h, \theta_h^*\rangle\right] = \mathbb{E}^\pi\left[R(s_h, a_h) + \rho(s_{h+1})\cdot\Theta + \langle \varphi(s_{h+1}), \theta_{h+1}^*\rangle\right].$$
	Telescoping this equation from $h = 1$ to $h = H$, we will get \eqref{eq-eq}.
\end{proof}

\par The following lemma indicates that if we have accurate enough estimation on the $\theta_1, \cdots, \theta_H$, then the greedy policy is a nearly optimal policy.
\begin{lemma}\label{lem2}
	Suppose we have estimation $\hat{\theta}_1, \cdots, \hat{\theta}_H$, which satisfy
	$$\left\|\hat{\theta}_h - \theta_h^*\right\|_2\le \varepsilon_h, \quad \forall 1\le h\le H.$$
	If consider policy $\pi$ to be the greedy policy with respect to $\hat{\theta}_1, \cdots, \hat{\theta}_H$, then we have, e.g.
	$$\pi(s_h) = \argmax_{a_h\in\mA(s_h)}\left\langle \varphi(s_h, a_h), \hat{\theta}_h\right\rangle,$$
	then we have
	$$V_1^*(s_1) - V_1^\pi(s_1)\le 2\sum_{h=1}^H \EE[\rho(s_h)]\varepsilon_h.$$
\end{lemma}
\begin{proof}
	We use $\pi^*(s)$ to denote the action of $s$ in the best policy (which is deterministic). Then we have
	$$V_1^*(s_1) = Q_1^*(s_1, \pi^*(s_1)) = \langle \varphi(s_1), \theta_1^*\rangle + \langle \pi^*(s_1), \theta_1^*\rangle.$$
	We further have
	\begin{align*}
		\langle \pi^*(s_1), \theta_1^*\rangle & = \left\langle \pi^*(s_1), \theta_1^* - \hat{\theta}_1\right\rangle + \left\langle \pi^*(s_1), \hat{\theta}_1\right\rangle\le \rho(s_1)\cdot \left\|\theta_1^* - \hat{\theta}_1\right\|_2 + \left\langle \pi^*(s_1), \hat{\theta}_1\right\rangle\\
		& \le \rho(s_1)\cdot\left\|\theta_1^* - \hat{\theta}_1\right\|_2 + \left\langle \pi(s_1), \hat{\theta}_1\right\rangle = 
\rho(s_1)\cdot\left\|\theta_1^* - \hat{\theta}_1\right\|_2 + \left\langle \pi(s_1), \theta_1^*\right\rangle + \left\langle \pi(s_1), \hat{\theta}_1 - \theta_1^*\right\rangle\\
		& \le 2\rho(s_1)\cdot\left\|\theta_1^* - \hat{\theta}_1\right\|_2 + \left\langle \pi(s_1), \theta_1^*\right\rangle,
	\end{align*}
	where the first and last inequality are due to $\pi(s_1), \pi^*(s_1)\in \mA(s_1)$, and the second inequality is due to the definition of $\pi(s_1)$ (greedy policy $w.r.t$ $\hat{\theta}_1$). Therefore, we obtain that
	\begin{align*}
		V_1^*(s_1) - V^\pi(s_1) & \le 2\rho(s_1)\cdot\left\|\theta_1^* - \hat{\theta}_1\right\|_2 + \left\langle \pi(s_1), \theta_1^*\right\rangle + \langle \varphi(s_1), \theta_1^*\rangle - V^\pi(s_1)\\
		& = 2\rho(s_1)\cdot\left\|\theta_1^* - \hat{\theta}_1\right\|_2 + Q^*_1(s_1, \pi(s_1)) - V^\pi(s_1)\\
		& = 2\rho(s_1)\cdot\left\|\theta_1^* - \hat{\theta}_1\right\|_2 + r(s_1, \pi(s_1)) + \mathbb{E}^\pi[V_2^*(s_2)] - \left(r(s_1, \pi(s_1)) + \mathbb{E}^\pi[V_2^\pi(s_2)]\right)\\
		& = 2\rho(s_1)\cdot\left\|\theta_1^* - \hat{\theta}_1\right\|_2 + \mathbb{E}^\pi\left[V_2^*(s_2) - V_2^\pi(s_2)\right].
	\end{align*}
	We can continue the same process for $V_2$, and obtain that
	$$\mathbb{E}^\pi\left[V_2^*(s_2) - V_2^\pi(s_2)\right]\le 2\mathbb{E}^\pi[\rho(s_2)]\cdot\left\|\theta_2^* - \hat{\theta}_2\right\|_2 + \mathbb{E}^\pi\left[V_3^*(s_3) - V_3^\pi(s_3)\right].$$
	Keeping doing this until we get $V_{H+1}^*(s_{H+1}) - V_{H+1}^\pi(s_{H+1})$, which is zero, we obtain that
	$$V_1^*(s_1) - V_1^\pi(s_1)\le 2\sum_{h=1}^H \EE[\rho(s_h)]\cdot\left\|\theta_h^* - \hat{\theta}_h\right\|_2.$$
	Therefore, when $\left\|\theta_h^* - \hat{\theta}_h\right\|_2\le \varepsilon_h$, we will have
	$$V_1^*(s_1) - V_1^\pi(s_1)\le 2\sum_{h=1}^H \EE[\rho(s_h)]\varepsilon_h$$
\end{proof}

\begin{proof}[Proof of Theorem \ref{thm1}]
	We will divide our proof in the following parts:
	\paragraph{Concentration Inequalities} In the following, we fix all parameters within one `while' loop in the algorithm.
	\par According to the Boundedness Assumption (Assumption \ref{ass-bound}), and Hoeffding inequality, we have that each of the following inequalities holds with probability at least $1 - \delta'$ separately for $0\le i\le d$,
	\begin{equation}\label{eq-9}\begin{aligned}
		& \left|R_{h, i} - \mathbb{E}^{\pi_{h, i}}\left[\sum_{h'=1}^H R(s_{h'}, a_{h'})\right]\right|\le \sqrt{\frac{2\log (1/\delta')}{M_1}}\\
		& \left|s_{h, i} - \mathbb{E}^{\pi_{h, i}}\left[\sum_{h'=1}^H \rho(s_{h'+1})\right]\right|\le \sqrt{\frac{2\log (1/\delta')}{M_1}}.\\
	\end{aligned}\end{equation}
	We further notice that the before step $h$, policy $\pi_{h, i}$ and policy $\pi_{h, 0}$ are both same as policy $\pi$, and for every $h' > h$ and $s_{h'}\in \mS_{h'}$, we both have $\pi_{h, i}(s_{h'}) = \pi_{h, 0}(s_{h'}) = 0$. Therefore, according to \eqref{eq-eq}, we have
	\begin{equation}\label{eq-10}\begin{aligned}
		\mathbb{E}^\pi[\rho(s_h)]\cdot\langle \mathbf{e}_i, \theta_h^*\rangle & = \EE^{\pi_{h, i}}\left[\sum_{h'=1}^H\langle a_{h'}, \theta_{h'}^*\rangle\right] - \EE^{\pi_{h, 0}}\left[\sum_{h'=1}^H\langle a_{h'}, \theta_{h'}^*\rangle\right]\\
		& = \mathbb{E}^{\pi_{h, i}}\left[\sum_{h'=1}^HR(s_{h'}, a_{h'})\right] + \Theta\cdot \mathbb{E}^{\pi_{h, i}}\left[\sum_{h'=1}^H \rho(s_{h'+1})\right]\\
		&\quad  - \mathbb{E}^{\pi_{h, 0}}\left[\sum_{h'=1}^HR(s_{h'}, a_{h'})\right] - \Theta\cdot \mathbb{E}^{\pi_{h, 0}}\left[\sum_{h'=1}^H \rho(s_{h'+1})\right].
	\end{aligned}\end{equation}
	According to \eqref{eq-9}, if we further assume $|\Theta - \xi|\le \eta$, then with probability at least $1 - 2(d+1)\delta'$, we have for every $1\le i\le d$, 
	$$|\text{RHS of } \eqref{eq-10} - ((s_{h, i} - s_{h, 0})\xi + R_{h, i} - R_{h, 0})|\le \eta + \sqrt{\frac{2\log (1/\delta')}{M_1}} + \sqrt{\frac{2\log (1/\delta')}{M_1}} = \eta + 2\sqrt{\frac{2\log (1/\delta')}{M_1}},$$
	where we have used the fact that $|s_{h, i} - s_{h, 0}|\le 1$ and also $|\Theta|\le 1.$ Moreover, in the last `while' loop, with probability at least $1 - LH\delta$, according to Hoeffding inequality we have
	$$\left|\rho_{h}^l - \mathbb{E}^{\pi_l'}\left[\rho(s_h)\right]\right|\le \sqrt{\frac{2\log (1/\delta')}{M_2}}, \quad \forall 1\le l\le L, 1\le h\le H.$$
	Hence we have with probability at least $1 - (2d+3)\delta'$,
	$$\left|\rho_h\langle \mathbf{e}_i, \theta_h^*\rangle - ((s_{h, i} - s_{h, 0})\xi + R_{h, i} - R_{h, 0})\right|\le \eta + 2\sqrt{\frac{2\log (1/\delta')}{M_1}} + \sqrt{\frac{2\log (1/\delta')}{M_2}}.$$
	According to our algorithm, there exists one $l$ such that $|l\eta - \Theta|\le \eta$, we write this $l$ as $l_0$. Then with $\xi = l_0\eta$, we have
	$$\hat{\theta}_{h, i}^{l_0} = \frac{(s_{h, i} - s_{h, 0})\xi + R_{h, i} - R_{h, 0}}{\rho_h}.$$
	Therefore, we obtain that
	$$\left|\rho_h\left(\langle \mathbf{e}_i, \theta_h^*\rangle - \hat{\theta}_{h, i}^{l_0}\right)\right|\le \eta + 2\sqrt{\frac{2\log (1/\delta')}{M_1}} + \sqrt{\frac{2\log (1/\delta')}{M_2}}.$$
	According to our construction of $\hat{\theta}_h^{l_0}$, we have
	$$\rho_h\left\|\theta_h - \hat{\theta}_{h}^{l_0}\right\|_2\le \eta d + 2d\sqrt{\frac{2\log (1/\delta')}{M_1}} + d\sqrt{\frac{2\log (1/\delta')}{M_2}}.$$
	\par If in the next `while' loop, the condition does not satisfy, then before the next loop, for every $1\le h\le H$ and $1\le l\le L$, we will have $\rho_h^l\le \varepsilon$ or $\rho_h^l\le \rho_h$. Hence for $l_0$, we also have either $\rho_h^{l_0}\le \varepsilon$ or $\rho_h^{l_0}\le \rho_h$ for $1\le h\le H$. With loss of generality, we can assume $\left\|\hat{\theta}_h^{l_0}\right\|_2\le 1$, otherwise we can consider $\hat{\theta}_h^{l_0}\leftarrow\hat{\theta}_h^{l_0}/\|\hat{\theta}_h^{l_0}\|_2$, which will induce the same policy, and $\left\|\theta_h - \hat{\theta}_{h}^{l_0}\right\|_2$ will decrease under this operation. According to Hoeffding inequality and union bound, we have with probability at least $1 - HL\delta'$, for every $1\le l\le L$ and $1\le h\le H$,
	$$\left|\rho_h^l - \mathbb{E}^{\pi_l'}(\rho(s_h))\right|\le \sqrt{\frac{2\log(1/\delta')}{M_2}}.$$
	
	Hence according to Lemma \ref{lem2}, we have
	$$V_1^*(s_1) - V_1^{\pi_{l_0}'}(s_1)\le 2\sum_{h=1}^H\EE^{\pi_{l_0}'}[\rho(s_h)]\left\|\theta_h - \hat{\theta}_{h}^{l_0}\right\|_2.$$
	If we assume the above high-probability event all hold, then we have
	$$\text{RHS}\le 4H\sqrt{\frac{2\log(1/\delta')}{M_2}} + 2\sum_{h=1}^H\rho_h^{l_0}\left\|\theta_h - \hat{\theta}_{h}^{l_0}\right\|_2,$$
	where we use the fact that $\left\|\theta_h - \hat{\theta}_{h}^{l_0}\right\|_2\le 2$ since $\|\theta_h\|_2\le 1$ and $\|\hat{\theta}_{h}^{l_0}\|_2\le 1$ both hold. According to the condition that we do not enter the next `while' loop, the last term above can be upper bounded by
	$$2H\cdot \left(\varepsilon + \eta d + 2d\sqrt{\frac{2\log (1/\delta')}{M_1}} + d\sqrt{\frac{2\log (1/\delta')}{M_2}}\right).$$
	Therefore, we obtain that with probability at least $1 - d\delta' - 2HL\delta'$
	$$V_1^*(s_1) - V_1^{\pi_{l_0}'}(s_1)\le 2H\varepsilon + 2H\eta d + 4H\sqrt{\frac{2\log(1/\delta')}{M_2}} + 4Hd\sqrt{\frac{2\log (1/\delta')}{M_1}} + 2Hd\sqrt{\frac{2\log (1/\delta')}{M_2}}.$$
	\par Further again according to Hoeffding inequality, we know that with probability at least $1 - HL\delta$, we have
	$$\left|R_l - V_1^{\pi_{l_0}'}(s_0)\right|\le \sqrt{\frac{2\log (1/\delta')}{M_2}}.$$
	If we assume
	$$l_1 = \argmax_{l} R_l,$$
	then we have
	$$V_1^{\pi_{l_0}'}(s_0) - V_1^\pi(s_0) = V_1^{\pi_{l_0}'}(s_0) - V_1^{\pi_{l_1}'}(s_0) \le 2\sqrt{\frac{2\log(1/\delta')}{M_2}} + R_{l_0} - R_{l_1}\le 2\sqrt{\frac{2\log(1/\delta')}{M_2}}.$$
	\par Therefore, we have proved that if the `while' loop ends, then with probability at least $1 - d\delta' - 3HL\delta'$, the output policy $\pi$ satisfies that
	$$V_1^*(s_1) - V_1^\pi(s_1)\le 2H\varepsilon + 2H\eta d + (2 + 4H + 2Hd)\sqrt{\frac{2\log(1/\delta')}{M_2}} + 4Hd\sqrt{\frac{2\log (1/\delta')}{M_1}}.$$
	
	\paragraph{Bound on Iterations in `While' Loop} What left in this proof is to show that the `while' loop will terminate within some certain number of iterations.
	\par We notice that starting from the third iterations of the `while' loop, if we choose some $h$ which satisfies the loop condition, then the value of $\rho_h$ will at least be twice of the previous loop. And the value of $\rho_h$ will be at least $\varepsilon$ in order to enter the loop. Moreover, the value of $\rho_h$ will never exceed $1$ due to the boundedness assumption. Therefore, the loop will at most run
	$$1 + H\log_2\left(\frac{1}{\varepsilon}\right)$$
	times. 
	\par Therefore, choosing $\varepsilon = \frac{\epsilon}{8H}$, $\delta' = \frac{\delta}{(d + 3HL)(1 + H\log_2(1/\varepsilon))}$, $\eta = \frac{\epsilon}{8Hd}$, $M_2 = 2\log(1/\delta')\cdot \frac{16(2 + 4H + 2Hd)^2}{\epsilon^2}$ and $M_1 = 2\log(1/\delta')\cdot \frac{256H^2d^2}{\epsilon^2}$ and $L = \frac{1}{\eta}$, then our algorithm will output an $\epsilon$-optimal policy using at most
	$$\left(M_1 Hd + M_2HL\right)\cdot \left(1 + H\log_2\left(\frac{1}{\varepsilon}\right)\right) = \tilde{\mathcal{O}}\left(\frac{H^5d^3}{\epsilon^3}\right)$$
	number of samples.
\end{proof}

\end{document}